\documentclass[pdflatex, sn-basic]{sn-jnl}
\usepackage{graphicx}%
\usepackage{multirow}%
\usepackage{amsmath,amssymb,amsfonts}%
\usepackage{amsthm}%
\usepackage{mathtools}%
\usepackage{mathrsfs}%
\usepackage[title]{appendix}%
\usepackage{xcolor}%
\usepackage{textcomp}%
\usepackage{manyfoot}%
\usepackage{booktabs}%
\usepackage{subcaption}%
\usepackage{pgfplots}%
\pgfplotsset{compat=1.18}%
\usepgfplotslibrary{groupplots}%
\usepackage{placeins}%
\usepackage{enumitem}%
\usepackage{url}%

\theoremstyle{thmstyleone}%
\newtheorem{theorem}{Theorem}%
\newtheorem{proposition}[theorem]{Proposition}%
\newtheorem{lemma}[theorem]{Lemma}%
\newtheorem{corollary}[theorem]{Corollary}%

\theoremstyle{thmstylethree}%
\newtheorem{definition}[theorem]{Definition}%

\theoremstyle{thmstyletwo}%
\newtheorem{remark}[theorem]{Remark}%

\definecolor{cinf}{HTML}{3B6EA8}
\definecolor{cunf}{HTML}{E8546B}

\def\datMnistinformed{(0.001,10.642) (1,9.425) (2,8.535) (3,7.849) (4,7.279) (5,6.781) (6,6.338) (7,5.942) (8,5.582) (9,5.253) (10,4.95) (11,4.671) (12,4.414) (13,4.175) (14,3.951) (15,3.743) (16,3.549) (17,3.368) (18,3.2) (19,3.042) (20,2.894) (21,2.755) (22,2.625) (23,2.502) (24,2.384) (25,2.274) (26,2.169) (27,2.072) (28,1.98) (29,1.896) (30,1.818) (31,1.748) (32,1.683) (33,1.625) (34,1.571) (35,1.52) (36,1.473) (37,1.427) (38,1.384) (39,1.346) (40,1.312) (41,1.281) (42,1.253) (43,1.228) (44,1.209) (45,1.192) (46,1.178) (47,1.167) (48,1.158) (49,1.152) (50,1.148) (51,1.143) (52,1.137) (53,1.128) (54,1.117) (55,1.104) (56,1.091) (57,1.078) (58,1.067) (59,1.057) (60,1.051) (61,1.045) (62,1.04) (63,1.033) (64,1.028) (65,1.024) (66,1.02) (67,1.014) (68,1.005) (69,0.994) (70,0.983) (71,0.971) (72,0.96) (73,0.951) (74,0.944) (75,0.939) (76,0.937) (77,0.935) (78,0.936) (79,0.937) (80,0.936) (81,0.936) (82,0.934) (83,0.931) (84,0.929) (85,0.926) (86,0.923) (87,0.92) (88,0.918) (89,0.916) (90,0.915) (91,0.915) (92,0.912) (93,0.912) (94,0.912) (95,0.911) (96,0.91) (97,0.909) (98,0.909) (99,0.91) (100,0.909) (101,0.908) (102,0.906) (103,0.906) (104,0.904) (105,0.903) (106,0.903) (107,0.903) (108,0.901) (109,0.899) (110,0.899) (111,0.899) (112,0.898) (113,0.898) (114,0.898) (115,0.897) (116,0.897) (117,0.897) (118,0.898) (119,0.898) (120,0.898)}
\def\datMnistuninformed{(0.001,10.642) (1,5.969) (2,4.023) (3,3.105) (4,2.738) (5,2.562) (6,2.451) (7,2.374) (8,2.317) (9,2.275) (10,2.234) (11,2.196) (12,2.166) (13,2.142) (14,2.119) (15,2.095) (16,2.079) (17,2.067) (18,2.053) (19,2.04) (20,2.018) (21,2.013) (22,2.006) (23,2.002) (24,1.998) (25,1.991) (26,1.98) (27,1.981) (28,1.977) (29,1.97) (30,1.959) (31,1.962) (32,1.96) (33,1.951) (34,1.95) (35,1.944) (36,1.957) (37,1.954) (38,1.949) (39,1.947) (40,1.942) (41,1.942) (42,1.941) (43,1.936) (44,1.928) (45,1.929) (46,1.933) (47,1.928) (48,1.927) (49,1.93) (50,1.925) (51,1.913) (52,1.905) (53,1.913) (54,1.91) (55,1.919) (56,1.91) (57,1.916) (58,1.908) (59,1.911) (60,1.913) (61,1.912) (62,1.905) (63,1.908) (64,1.898) (65,1.895) (66,1.902) (67,1.908) (68,1.904) (69,1.899) (70,1.901) (71,1.9) (72,1.899) (73,1.906) (74,1.909) (75,1.91) (76,1.902) (77,1.896) (78,1.895) (79,1.892) (80,1.894) (81,1.902) (82,1.895) (83,1.901) (84,1.895) (85,1.891) (86,1.893) (87,1.899) (88,1.902) (89,1.904) (90,1.908) (91,1.911) (92,1.911) (93,1.918) (94,1.912) (95,1.917) (96,1.916) (97,1.913) (98,1.911) (99,1.906) (100,1.899) (101,1.912) (102,1.903) (103,1.902) (104,1.897) (105,1.886) (106,1.88) (107,1.881) (108,1.877) (109,1.891) (110,1.89) (111,1.9) (112,1.902) (113,1.899) (114,1.896) (115,1.895) (116,1.892) (117,1.895) (118,1.899) (119,1.904) (120,1.911)}
\def\datFmnistinformed{(0.001,10.64) (1,10.293) (2,9.997) (3,9.716) (4,9.456) (5,9.212) (6,8.979) (7,8.76) (8,8.552) (9,8.355) (10,8.17) (11,7.995) (12,7.829) (13,7.671) (14,7.519) (15,7.373) (16,7.234) (17,7.099) (18,6.969) (19,6.844) (20,6.723) (21,6.606) (22,6.492) (23,6.382) (24,6.276) (25,6.173) (26,6.074) (27,5.977) (28,5.884) (29,5.793) (30,5.706) (31,5.62) (32,5.537) (33,5.457) (34,5.38) (35,5.304) (36,5.231) (37,5.16) (38,5.092) (39,5.024) (40,4.959) (41,4.896) (42,4.834) (43,4.773) (44,4.715) (45,4.658) (46,4.602) (47,4.547) (48,4.493) (49,4.442) (50,4.39) (51,4.34) (52,4.291) (53,4.243) (54,4.197) (55,4.151) (56,4.106) (57,4.062) (58,4.02) (59,3.98) (60,3.94) (61,3.901) (62,3.863) (63,3.825) (64,3.789) (65,3.754) (66,3.719) (67,3.683) (68,3.65) (69,3.618) (70,3.585) (71,3.553) (72,3.522) (73,3.491) (74,3.46) (75,3.43) (76,3.4) (77,3.371) (78,3.342) (79,3.315) (80,3.287) (81,3.261) (82,3.235) (83,3.21) (84,3.185) (85,3.161) (86,3.138) (87,3.114) (88,3.091) (89,3.069) (90,3.047) (91,3.026) (92,3.005) (93,2.985) (94,2.965) (95,2.945) (96,2.926) (97,2.906) (98,2.887) (99,2.87) (100,2.852) (101,2.834) (102,2.817) (103,2.801) (104,2.785) (105,2.769) (106,2.754) (107,2.738) (108,2.724) (109,2.708) (110,2.693) (111,2.679) (112,2.664) (113,2.649) (114,2.634) (115,2.62) (116,2.606) (117,2.592) (118,2.58) (119,2.567) (120,2.555)}
\def\datFmnistuninformed{(0.001,10.64) (1,9.426) (2,8.595) (3,7.911) (4,7.361) (5,6.903) (6,6.522) (7,6.201) (8,5.934) (9,5.717) (10,5.524) (11,5.353) (12,5.215) (13,5.111) (14,5.027) (15,4.951) (16,4.878) (17,4.816) (18,4.766) (19,4.728) (20,4.679) (21,4.642) (22,4.611) (23,4.596) (24,4.564) (25,4.545) (26,4.524) (27,4.506) (28,4.489) (29,4.489) (30,4.478) (31,4.454) (32,4.438) (33,4.424) (34,4.419) (35,4.415) (36,4.41) (37,4.395) (38,4.395) (39,4.393) (40,4.376) (41,4.373) (42,4.361) (43,4.37) (44,4.359) (45,4.354) (46,4.352) (47,4.344) (48,4.341) (49,4.345) (50,4.345) (51,4.341) (52,4.344) (53,4.341) (54,4.336) (55,4.344) (56,4.348) (57,4.346) (58,4.343) (59,4.341) (60,4.336) (61,4.332) (62,4.329) (63,4.326) (64,4.324) (65,4.319) (66,4.31) (67,4.306) (68,4.296) (69,4.298) (70,4.295) (71,4.289) (72,4.291) (73,4.285) (74,4.286) (75,4.291) (76,4.284) (77,4.291) (78,4.303) (79,4.3) (80,4.296) (81,4.293) (82,4.294) (83,4.297) (84,4.284) (85,4.287) (86,4.286) (87,4.283) (88,4.282) (89,4.275) (90,4.291) (91,4.285) (92,4.277) (93,4.288) (94,4.299) (95,4.284) (96,4.283) (97,4.291) (98,4.287) (99,4.279) (100,4.29) (101,4.3) (102,4.293) (103,4.294) (104,4.296) (105,4.285) (106,4.289) (107,4.293) (108,4.284) (109,4.293) (110,4.298) (111,4.303) (112,4.31) (113,4.3) (114,4.304) (115,4.312) (116,4.319) (117,4.333) (118,4.328) (119,4.319) (120,4.315)}
\def\datCifarinformed{(0.001,10.894) (1,10.37) (2,9.882) (3,9.448) (4,9.072) (5,8.727) (6,8.414) (7,8.127) (8,7.861) (9,7.613) (10,7.379) (11,7.158) (12,6.951) (13,6.756) (14,6.57) (15,6.394) (16,6.226) (17,6.068) (18,5.917) (19,5.773) (20,5.636) (21,5.505) (22,5.38) (23,5.26) (24,5.144) (25,5.032) (26,4.926) (27,4.826) (28,4.73) (29,4.638) (30,4.549) (31,4.465) (32,4.384) (33,4.305) (34,4.229) (35,4.156) (36,4.086) (37,4.019) (38,3.955) (39,3.891) (40,3.832) (41,3.774) (42,3.72) (43,3.667) (44,3.618) (45,3.569) (46,3.523) (47,3.479) (48,3.434) (49,3.392) (50,3.351) (51,3.311) (52,3.274) (53,3.238) (54,3.205) (55,3.174) (56,3.144) (57,3.114) (58,3.086) (59,3.059) (60,3.034) (61,3.011) (62,2.988) (63,2.965) (64,2.943) (65,2.922) (66,2.902) (67,2.883) (68,2.865) (69,2.848) (70,2.831) (71,2.815) (72,2.799) (73,2.785) (74,2.772) (75,2.76) (76,2.75) (77,2.738) (78,2.725) (79,2.712) (80,2.699) (81,2.687) (82,2.676) (83,2.666) (84,2.657) (85,2.647) (86,2.639) (87,2.632) (88,2.623) (89,2.616) (90,2.609) (91,2.603) (92,2.596) (93,2.589) (94,2.582) (95,2.576) (96,2.572) (97,2.564) (98,2.557) (99,2.552) (100,2.546) (101,2.539) (102,2.533) (103,2.526) (104,2.521) (105,2.517) (106,2.514) (107,2.512) (108,2.51) (109,2.506) (110,2.501) (111,2.495) (112,2.488) (113,2.485) (114,2.484) (115,2.48) (116,2.475) (117,2.468) (118,2.463) (119,2.459) (120,2.457)}
\def\datCifaruninformed{(0.001,10.894) (1,9.497) (2,8.577) (3,7.863) (4,7.291) (5,6.832) (6,6.443) (7,6.118) (8,5.832) (9,5.61) (10,5.426) (11,5.259) (12,5.125) (13,5.024) (14,4.928) (15,4.846) (16,4.779) (17,4.717) (18,4.668) (19,4.622) (20,4.584) (21,4.556) (22,4.522) (23,4.496) (24,4.465) (25,4.455) (26,4.432) (27,4.413) (28,4.396) (29,4.386) (30,4.375) (31,4.353) (32,4.342) (33,4.331) (34,4.327) (35,4.316) (36,4.303) (37,4.3) (38,4.292) (39,4.288) (40,4.289) (41,4.285) (42,4.282) (43,4.274) (44,4.271) (45,4.265) (46,4.258) (47,4.256) (48,4.257) (49,4.25) (50,4.246) (51,4.247) (52,4.239) (53,4.259) (54,4.263) (55,4.265) (56,4.267) (57,4.266) (58,4.263) (59,4.258) (60,4.254) (61,4.25) (62,4.239) (63,4.237) (64,4.236) (65,4.229) (66,4.234) (67,4.229) (68,4.236) (69,4.233) (70,4.232) (71,4.236) (72,4.237) (73,4.226) (74,4.225) (75,4.214) (76,4.206) (77,4.209) (78,4.204) (79,4.202) (80,4.208) (81,4.206) (82,4.203) (83,4.21) (84,4.204) (85,4.204) (86,4.204) (87,4.207) (88,4.201) (89,4.199) (90,4.192) (91,4.193) (92,4.191) (93,4.195) (94,4.191) (95,4.189) (96,4.195) (97,4.189) (98,4.181) (99,4.17) (100,4.18) (101,4.182) (102,4.182) (103,4.178) (104,4.177) (105,4.176) (106,4.177) (107,4.175) (108,4.184) (109,4.189) (110,4.183) (111,4.183) (112,4.181) (113,4.182) (114,4.18) (115,4.19) (116,4.182) (117,4.184) (118,4.189) (119,4.184) (120,4.183)}

\begin{document}

\title{Solvable High-Dimensional GAN Dynamics under Structured and Mismatched Second Moments}

\author[1]{\fnm{Andrew} \sur{Bond}}\email{abond19@ku.edu.tr}

\author*[1,2]{\fnm{Zafer} \sur{Do\u{g}an}}\email{zdogan@ku.edu.tr}

\affil[1]{\orgdiv{MLIP Group, KUIS AI Center}, \orgname{Ko\c{c} University}, \orgaddress{\city{\.{I}stanbul}, \country{Turkey}}}

\affil[2]{\orgdiv{Dept. of EE Engineering}, \orgname{Ko\c{c} University}, \orgaddress{\city{\.{I}stanbul}, \country{Turkey}}}

\abstract{
We study the exact high-dimensional dynamics of a solvable
multi-feature GAN in which a linear generator learns a low-dimensional
subspace by sample-wise stochastic training. Extending prior analyses
with unconditional diagonal latent covariance, we allow
class-structured mixtures with correlations and non-zero means. For the
quadratic energy discriminator, this heterogeneity enters through
probability-weighted effective second moments. We prove that the
macroscopic observables concentrate onto a deterministic ODE in the
high-dimensional limit. For arbitrary mismatched true and generator
moments, we derive an exact transverse stability criterion through a
$2d\times2d$ block Jacobian. In the commuting case, it reduces to
explicit mode-wise conditions and recovers the matched solvable
interval as a corollary; in the noncommuting case, stability also
depends on relative eigenspace alignment. The analysis further shows
that generator underestimation can destabilize recovery and create
partial-alignment equilibria. Low-rank correlations can boost
individually weak coordinates above the learning threshold, although
excessive amplification can destabilize recovery in the matched
setting. Simulations support the ODE, spectral onset, and
effective-second-moment equivalence. Image experiments on MNIST,
FashionMNIST, and CIFAR-10 further demonstrate finite-time sensitivity
to the generator moment, with informed class-specific moments improving
reference-subspace alignment within the solvable linear model.
}

\maketitle

\section{Introduction}\label{intro}

Generative adversarial networks (GANs)~\citep{gan} are trained through a coupled minimax game, making their dynamics substantially harder to analyze than empirical-risk minimization: even simple games may oscillate, destabilize, or collapse. A large literature studies stabilization and convergence~\citep{wgan,wgan_gp,mescheder2018converge,unrolled_gan,nagarajan2017gradient,ttur}, but exact descriptions of the sample-wise stochastic training trajectory remain rare. A complementary approach studies solvable high-dimensional models in which the microscopic weights remain stochastic while a finite collection of macroscopic observables converges to deterministic dynamics. In this direction, \citet{solvable} introduced a single-layer GAN exhibiting distinct regimes of recovery, oscillation, failure, and mode collapse, while \citet{multifeaturegan} extended the framework to multi-feature discriminators, where simultaneous discriminator features induce coupled rather than sequential subspace learning.

This paper studies a deliberately minimal linear-generator, quadratic-discriminator model. Our objective is not to propose a competitive practical GAN architecture, but to use an analytically controlled adversarial system in which the complete sample-wise stochastic training trajectory can be characterized in the high-dimensional limit. The model retains the essential coupled generator--discriminator dynamics and a nontrivial subspace geometry while permitting explicit learning and stability conditions. Accordingly, the mechanisms identified here are rigorous properties of the solvable model and hypotheses to be investigated in deeper architectures, rather than claims already established for modern GANs.

Existing closed-form analyses of this model class assume a fixed, unconditional, diagonal latent covariance, so the latent coordinates behave as independent modes. We generalize the data model to a class-structured mixture with class-dependent feature strengths, cross-feature correlations, and non-zero class means. Labels index the mixture components but are not supplied to the discriminator; training therefore remains an unconditional adversarial game on class-structured data rather than a conventional conditional GAN. Such structure is nevertheless motivated by conditional generation~\citep{cgan}, latent representation learning~\citep{infogan}, and image generators whose latent spaces encode correlated factors~\citep{stylegan}. Our central question is how this structure---and especially mismatch between the true and generated latent statistics---changes the exact high-dimensional dynamics of online adversarial training.

The first simplification follows from the quadratic energy discriminator. At the population level, its objective depends on the data only through $\mathbb{E}[\mathbf{y}\mathbf{y}^{\top}]$. Consequently, class-dependent covariances, correlations, and means enter through the effective second moments
$\bar{\Lambda}=\mathbb{E}[\mathbf{c}\mathbf{c}^{\top}]$ and
$\bar{\widetilde{\Lambda}}=\mathbb{E}[\widetilde{\mathbf{c}}\widetilde{\mathbf{c}}^{\top}]$.
This immediately implies an identifiability boundary: latent mixtures sharing the same effective second moment are indistinguishable to the quadratic adversarial dynamics. The nontrivial result is that, under this label-mixed, correlated, and non-zero-mean latent law, the macroscopic observables of the microscopic stochastic process still concentrate on a closed deterministic system of macroscopic ordinary differential equations.

A single positive-semidefinite second moment can always be diagonalized by a change of latent basis. Correlation in the true distribution therefore modifies the matched dynamics through its effective eigenvalues, but does not by itself create an irreducible new geometry. The genuinely new behavior arises when
$\bar{\Lambda}\neq\bar{\widetilde{\Lambda}}$: no single change of basis generally diagonalizes both matrices, and the relative alignment of their eigenspaces affects stability. We show that the equilibria form a manifold, derive the exact transverse Jacobian for arbitrary possibly noncommuting second moments and unequal noise levels, and obtain explicit mode-wise conditions when the two moments commute. These results separate the onset of learning, which is governed by the true effective second moment, from recovery and equilibrium selection, which depend jointly on the true and generator moments. They also identify mismatch-induced partial-alignment equilibria and show that perfect recovery can become unstable even when the spectra of both moments are held fixed and only their relative eigenbasis is rotated.

Within the matched specialization, the theory recovers a mode-wise solvable interval determined by the learning rates and noise level. This gives a complementary signal-boosting consequence: a low-rank correlation component can raise the leading effective eigenvalue above the learning threshold even when every marginal coordinate is individually weak, whereas excessive signal strength can destabilize recovery. Our contributions are summarized as follows.

\begin{itemize}[nosep, leftmargin=*]
    \item \textit{Exact dynamics under structured mixtures.}
    We derive a deterministic macroscopic ODE for sample-wise GAN training with class-dependent, correlated, and non-zero-mean latent structure, and characterize the information identifiable through effective second moments.
    \item \textit{Dynamics under second-moment mismatch.}
    For arbitrary true and generator second moments, we derive the transverse Jacobian, obtain explicit recovery conditions in the commuting case, and identify mismatch-induced partial alignment and sensitivity to relative eigenspace orientation.
    \item \textit{Controlled numerical validation.}
    Synthetic experiments validate the ODE, stability boundaries, and mixture equivalence. Image experiments evaluate informed generator moments against principled controls using held-out subspace and covariance-fit criteria.
\end{itemize}

\textbf{Notation.}
Bold capitals denote matrices and bold lowercase letters denote vectors. We write $\mathbf{A}^{\top}$ for transpose, $\operatorname{diag}(\mathbf{A})$ for the diagonal matrix formed from the diagonal entries of $\mathbf{A}$, and $\mathbf{I}$ for the identity. The true, generator, and discriminator subspaces are represented by $\mathbf{U}$, $\mathbf{V}$, and $\mathbf{W}$, respectively. We study the high-dimensional regime $n\to\infty$ with fixed intrinsic dimension $d$ and rescaled time $t=k/n$.

\section{Related Work}\label{sec:related_work}

\textbf{GAN optimization and training dynamics.}
The instability of GAN training has motivated alternative objectives, regularization, and game-dynamical algorithms, including Wasserstein objectives~\citep{wgan,wgan_gp}, local convergence analyses~\citep{nagarajan2017gradient,mescheder2018converge}, two-timescale updates~\citep{ttur}, and unrolled optimization~\citep{unrolled_gan}. These works study convergence or stabilization across broad model classes. Our objective is complementary: to derive the exact high-dimensional trajectory of sample-wise stochastic training in a solvable adversarial model.

\noindent\textbf{Solvable high-dimensional GANs and the LQG setting.}
\citet{solvable} showed that, for a single-layer GAN, macroscopic overlaps converge to deterministic ODEs while microscopic weights remain stochastic, enabling a dynamical classification of recovery, oscillation, failure, and mode collapse. \citet{multifeaturegan} extended this framework to multi-feature discriminators and established coupled, simultaneous subspace-learning dynamics. We build on that model by introducing class-structured mixtures, correlations, non-zero means, and mismatched true and generator second moments. Closest in architecture is~\citet{feizi2020lqg}, who connect linear--quadratic Gaussian GANs to covariance and PCA recovery and study population generalization and optimization stability. Our object is instead the high-dimensional scaling limit of the sample-wise stochastic trajectory. This yields learning thresholds, mismatch-dependent recovery conditions, equilibrium selection, and partial-alignment states that are not characterized by the population optimum.

\noindent\textbf{Structured latent variables and subspace models.}
Conditional GANs~\citep{cgan}, InfoGAN~\citep{infogan}, and style-based generators~\citep{stylegan} demonstrate the importance of class information and structured latent factors. Our labels instead index mixture components and are not supplied to the discriminator; the question is how such class-structured data affect solvable, label-blind adversarial dynamics. The resulting effective-second-moment model is related to spiked covariance models for high-dimensional low-rank recovery~\citep{spiked_covariance}, but the subspace is learned through adversarial interaction rather than direct sample-covariance estimation. The spectrum governs learning onset, while mismatch between the true and generator moments governs recovery and equilibrium selection.

\section{Problem Setup}\label{sect:background}

We study sample-wise adversarial training in the high-dimensional regime
$n\to\infty$, with fixed latent dimension $d$ and rescaled time
$t=k/n$. The model extends the solvable multi-feature GAN
of~\citet{multifeaturegan} by allowing structured latent mixtures and
mismatched true and generator second moments.

\subsection{Structured spiked data model}

At iteration $k$, the real sample is
\begin{equation}
    \mathbf y_k
    =
    \mathbf U\mathbf c_k
    +
    \sqrt{\eta_{\mathrm T}}\,\mathbf a_k,
    \qquad
    \mathbf a_k\sim\mathcal N(\mathbf 0,\mathbf I_n),
    \label{eq:true_data_model}
\end{equation}
where $\mathbf U\in\mathbb R^{n\times d}$ has orthonormal columns,
$\mathbf c_k\in\mathbb R^d$ contains the latent signal coordinates,
and $\eta_{\mathrm T}\geq 0$ is the real-data noise variance.

To model class-dependent feature strengths, correlations, and non-zero
means, let $\ell_k\in\{1,\ldots,L\}$ be drawn with
$\mathbb P(\ell_k=\ell)=\pi_\ell$, and conditionally sample
\begin{equation}
    \mathbf c_k\mid \ell_k=\ell
    \sim
    \mathcal N\!\left(
        \mathbf m_\ell,\,
        \boldsymbol{\Lambda}^{\ell}
        +
        \theta_c^\ell
        \boldsymbol{\gamma}_c^\ell
        (\boldsymbol{\gamma}_c^\ell)^\top
    \right).
    \label{eq:conditional_latent}
\end{equation}
Here $\boldsymbol{\Lambda}^{\ell}\succeq 0$ specifies the within-class
covariance, while the rank-one term introduces cross-coordinate
correlation. The resulting uncentered latent second moment is
\begin{equation}
    \bar{\boldsymbol{\Lambda}}
    :=
    \mathbb E[\mathbf c_k\mathbf c_k^\top]
    =
    \sum_{\ell=1}^{L}\pi_\ell
    \left(
        \boldsymbol{\Lambda}^{\ell}
        +
        \theta_c^\ell
        \boldsymbol{\gamma}_c^\ell
        (\boldsymbol{\gamma}_c^\ell)^\top
        +
        \mathbf m_\ell\mathbf m_\ell^\top
    \right).
    \label{eq:effective_cov_setup}
\end{equation}
Thus, both within-class correlations and non-zero class means contribute
to the same effective second moment. The labels index the sampling
mixture but are not supplied to either player; the model is therefore
class-structured rather than a conventional label-conditioned GAN.

\subsection{Generator and energy discriminator}

The generated sample is
\begin{equation}
    \widetilde{\mathbf y}_k
    =
    \mathbf V_k\widetilde{\mathbf c}_k
    +
    \sqrt{\eta_{\mathrm G}}\,\widetilde{\mathbf a}_k,
    \qquad
    \widetilde{\mathbf a}_k\sim\mathcal N(\mathbf 0,\mathbf I_n),
    \label{eq:generator_model}
\end{equation}
where $\mathbf V_k\in\mathbb R^{n\times d}$ is the generator and
$\widetilde{\mathbf c}_k$ has effective second moment
\begin{equation}
    \bar{\widetilde{\boldsymbol{\Lambda}}}
    :=
    \mathbb E[
        \widetilde{\mathbf c}_k
        \widetilde{\mathbf c}_k^\top
    ].
    \label{eq:generator_effective_cov}
\end{equation}
We do not assume
$\bar{\widetilde{\boldsymbol{\Lambda}}}
 =\bar{\boldsymbol{\Lambda}}$.
This mismatch is central to the stability and equilibrium-selection
analysis below.

The discriminator is represented by
$\mathbf W_k\in\mathbb R^{n\times d}$ and assigns the quadratic
projection energy
\begin{equation}
    D(\mathbf y;\mathbf W_k)
    =
    \|\mathbf W_k^\top\mathbf y\|^2.
    \label{eq:energy_discriminator}
\end{equation}
For a real--generated sample pair, the sample loss is
\begin{equation}
    \mathcal L_k(\mathbf V,\mathbf W)
    =
    \frac{1}{2}\|\mathbf W^\top\mathbf y_k\|^2
    -
    \frac{1}{2}\|\mathbf W^\top\widetilde{\mathbf y}_k\|^2.
    \label{eq:gan_loss}
\end{equation}
The discriminator performs ascent and the generator performs descent, with learning rates $\tau/n$ and $\widetilde{\tau}/n$:
\begin{equation}
\begin{aligned}
\mathbf V_{k+1}
&=\operatorname{Retr}_{\mathrm{col}}\!\left(
\mathbf V_k-\frac{\widetilde{\tau}}{n}\nabla_{\mathbf V}\mathcal L_k
\right),&
\mathbf W_{k+1}
&=\operatorname{Retr}_{\mathrm{St}}\!\left(
\mathbf W_k+\frac{\tau}{n}\nabla_{\mathbf W}\mathcal L_k
\right).
\end{aligned}
\label{eq:sgda_updates}
\end{equation}
Here
$\operatorname{Retr}_{\mathrm{col}}(\mathbf X)
:=\mathbf X[\operatorname{diag}(\mathbf X^\top\mathbf X)]^{-1/2}$
normalizes each column, whereas
$\operatorname{Retr}_{\mathrm{St}}(\mathbf X)
:=\mathbf X(\mathbf X^\top\mathbf X)^{-1/2}$
is the Stiefel polar retraction, with the principal inverse square root.
Consequently,
$\operatorname{diag}(\mathbf V_k^\top\mathbf V_k)=\mathbf I_d$ and
$\mathbf W_k^\top\mathbf W_k=\mathbf I_d$, while the off-diagonal entries
of $\mathbf V_k^\top\mathbf V_k$ remain dynamical. These retractions are
used in both the analysis and implementation.

\subsection{Effective-second-moment reduction}
\label{subsec:second_moment}

Because the payoff in~\eqref{eq:gan_loss} is quadratic, its conditional
drift depends on the latent laws through
\begin{equation}
    \mathbb E[\mathbf c_k\mathbf c_k^\top]
    =\bar{\boldsymbol{\Lambda}},
    \qquad
    \mathbb E[
        \widetilde{\mathbf c}_k
        \widetilde{\mathbf c}_k^\top
    ]
    =\bar{\widetilde{\boldsymbol{\Lambda}}},
    \label{eq:sample_second_moments}
\end{equation}
rather than through their mixture decompositions separately.
Theorem~\ref{thm:main_ode} shows that this reduction persists at the
level of the high-dimensional sample-wise stochastic dynamics: the
macroscopic overlaps of the microscopic iterates concentrate on a
closed deterministic system whose
distributional inputs are these two effective second moments.

\begin{remark}[Role of the quadratic energy]
\label{rem:degree}
The reduction to a single second moment is specific to the quadratic
payoff. Higher-order discriminators generally depend on higher-order
mixture moments and can distinguish distributions sharing the same
$\bar{\boldsymbol{\Lambda}}$. Whether such objectives admit a comparably
low-dimensional dynamical closure is outside the scope of this work.
\end{remark}

\begin{remark}[Identifiability]
\label{rem:identifiability}
Any two latent mixtures with the same
$\bar{\boldsymbol{\Lambda}}$ induce identical limiting dynamics under
the present discriminator. Hence class means, within-class covariances,
and mixture weights are not separately identifiable. This is both a
structural limitation of second-moment adversarial training and a
testable prediction, examined in Section~\ref{sec:experiments}.
\end{remark}

\subsection{Informed and Baseline Generator Moments}
\label{subsec:informed}

We compare generator-moment constructions through $\bar{\widetilde{\boldsymbol{\Lambda}}}$. The \emph{informed} construction uses class-dependent moments estimated from training data. The baselines remove or perturb different forms of this information: matched-trace isotropic, shuffled-spectrum, shrunk-spectrum, uniform, and class-agnostic global-PCA moments. The theory predicts how these choices affect transient learning and the transverse stability of the equilibrium manifold. In the
real-data experiments, they are evaluated against a PCA reference estimated from held-out data.

\section{High-Dimensional Training Dynamics}\label{sect:training_dynamics}

We now characterize the macroscopic training dynamics. Following~\citet{solvable,multifeaturegan}, we track low-dimensional overlaps among the true, generator, and discriminator subspaces and prove that these overlaps converge to a deterministic ODE as $n\to\infty$.

\subsection{Microscopic and Macroscopic States}

\begin{definition}[Microscopic state]
    The concatenated matrix
    \[
        \mathbf{X}_k \vcentcolon= \big[\,\mathbf{U},\,\mathbf{V}_k,\,\mathbf{W}_k\,\big] \in \mathbb{R}^{n\times 3d}
    \]
    is called the \emph{microscopic state} of the training process at iteration $k$.
\end{definition}

\begin{definition}[Macroscopic state]
    The tuple $\{\mathbf{P}_k,\mathbf{Q}_k,\mathbf{R}_k,\mathbf{S}_k,\mathbf{Z}_k\}$ is called the \emph{macroscopic state} at time $k$, where
    \[
        \mathbf{P}_k \vcentcolon= \mathbf{U}^\top\mathbf{V}_k, \quad
        \mathbf{Q}_k \vcentcolon= \mathbf{U}^\top\mathbf{W}_k, \quad
        \mathbf{R}_k \vcentcolon= \mathbf{V}_k^\top\mathbf{W}_k,
    \]
    \[
        \mathbf{S}_k \vcentcolon= \mathbf{V}_k^\top\mathbf{V}_k, \quad
        \mathbf{Z}_k \vcentcolon= \mathbf{W}_k^\top\mathbf{W}_k.
    \]
    Equivalently,
    \begin{equation}
        \mathbf{M}_k
        =
        \mathbf{X}_k^\top\mathbf{X}_k
        =
        \begin{bmatrix}
            \mathbf{I} & \mathbf{P}_k & \mathbf{Q}_k \\
            \mathbf{P}_k^\top & \mathbf{S}_k & \mathbf{R}_k \\
            \mathbf{Q}_k^\top & \mathbf{R}_k^\top & \mathbf{Z}_k
        \end{bmatrix}
        \in\mathbb{R}^{3d\times 3d}.
        \label{eq:macro_state_matrix}
    \end{equation}
\end{definition}

The matrices $\mathbf{P}_k$ and $\mathbf{Q}_k$ measure alignment of the generator and discriminator with the true subspace, while $\mathbf{R}_k$ measures generator--discriminator alignment and $\mathbf{S}_k,\mathbf{Z}_k$ encode self-correlations.

\subsection{Assumptions and Effective Second Moments}
\label{subsec:assumptions_effective_cov}

We use the following assumptions, following~\citet{multifeaturegan} except for the generalized latent law and the explicit constrained updates.
\begin{enumerate}
    \item[(A.1)] The latent sequences $\{\mathbf c_k\}$ and $\{\widetilde{\mathbf c}_k\}$ are mutually independent and i.i.d., with bounded moments of all orders. In the structured case, $\mathbf c_k$ follows~\eqref{eq:conditional_latent}; the generator satisfies $\mathbb E[\widetilde{\mathbf c}_k\widetilde{\mathbf c}_k^\top] =\bar{\widetilde{\boldsymbol{\Lambda}}}<\infty$.
    \item[(A.2)] The noise sequences $\{\mathbf a_k\}$ and $\{\widetilde{\mathbf a}_k\}$ are mutually independent and i.i.d.\ $\mathcal N(\mathbf0,\mathbf I_n)$, and are independent of all latent variables.
    \item[(A.3)] The sample loss is the quadratic energy
    in~\eqref{eq:gan_loss}, and the updates use the retractions defined
    after~\eqref{eq:sgda_updates}.
    \item[(A.4)] The initial state has bounded fourth moments: $\mathbb{E}[\sum_{\ell=1}^d ([\mathbf{U}]_{i,\ell}^4+[\mathbf{V}_0]_{i,\ell}^4+[\mathbf{W}_0]_{i,\ell}^4)]\le C/n^2$ for all $i$, with $C$ independent of $n$.
    \item[(A.5)] The initial macroscopic state concentrates: $\mathbb{E}\|\mathbf{M}_0-\mathbf{M}_0^\ast\|\le C/\sqrt{n}$ for some deterministic $\mathbf{M}_0^\ast$.
    \item[(A.6)] The discriminator columns are orthonormal, $\mathbf{W}_k^\top\mathbf{W}_k=\mathbf{I}_d$, so $\mathbf{Z}_k\equiv\mathbf{I}_d$.
\end{enumerate}

Assumptions (A.1)--(A.2) specify the sampling model, with (A.1) extending beyond Gaussian mixtures to latent laws satisfying the same moment conditions. Assumption (A.3) is structural and drives the second-moment reduction; (A.4)--(A.5) are the delocalization and initialization conditions used in the scaling-limit proof; and (A.6) fixes the discriminator constraint.

We write
\begin{equation}
\bar{\boldsymbol{\Lambda}}
:=\mathbb E[\mathbf c_k\mathbf c_k^\top],
\qquad
\bar{\widetilde{\boldsymbol{\Lambda}}}
:=\mathbb E[\widetilde{\mathbf c}_k
\widetilde{\mathbf c}_k^\top].
\label{eq:effective_cov}
\end{equation}
The unconditional diagonal model is recovered by setting
$\bar{\boldsymbol{\Lambda}}=\boldsymbol{\Lambda}$ and
$\bar{\widetilde{\boldsymbol{\Lambda}}}
=\widetilde{\boldsymbol{\Lambda}}$.

\subsection{Macroscopic ODE: Main Theorem}

\begin{theorem}[Macroscopic ODE limit]
\label{thm:main_ode}
Fix $T>0$. Under Assumptions (A.1)--(A.6), consider the constrained
algorithm of Section~\ref{sect:background}, which orthonormalizes
$\mathbf W$ and renormalizes the columns of $\mathbf V$ after each
update. Then
\begin{equation}
    \max_{0\leq k\leq nT}
    \mathbb{E}
    \left\|
        \mathbf M_k-\mathbf M(k/n)
    \right\|_{\mathrm F}
    \leq
    \frac{C(T)}{\sqrt n},
    \label{eq:ode_convergence}
\end{equation}
where $C(T)$ is independent of $n$ and
$\mathbf M(t)\in\mathbb R^{3d\times3d}$ is deterministic. Moreover,
$\mathbf M(t)$ is the unique solution of
\begin{equation}
\label{eq:ode_limit}
\begin{split}
    \frac{d}{dt}\mathbf P_t
    &=
    \widetilde{\tau}
    \left(
        \mathbf Q_t\mathbf R_t^\top
        \bar{\widetilde{\boldsymbol{\Lambda}}}
        +\mathbf P_t\mathbf L_t
    \right),\\
    \frac{d}{dt}\mathbf Q_t
    &=
    \tau
    \left(
        \bar{\boldsymbol{\Lambda}}\mathbf Q_t
        -\mathbf P_t
        \bar{\widetilde{\boldsymbol{\Lambda}}}\mathbf R_t
        +\mathbf Q_t\mathbf H_t
    \right),\\
    \frac{d}{dt}\mathbf R_t
    &=
    \tau
    \left(
        \mathbf P_t^\top
        \bar{\boldsymbol{\Lambda}}\mathbf Q_t
        -\mathbf S_t
        \bar{\widetilde{\boldsymbol{\Lambda}}}\mathbf R_t
        +\mathbf R_t\mathbf H_t
    \right)
    +
    \widetilde{\tau}
    \left(
        \bar{\widetilde{\boldsymbol{\Lambda}}}
        +\mathbf L_t
    \right)\mathbf R_t,\\
    \frac{d}{dt}\mathbf S_t
    &=
    \widetilde{\tau}
    \left(
        \mathbf R_t\mathbf R_t^\top
        \bar{\widetilde{\boldsymbol{\Lambda}}}
        +\bar{\widetilde{\boldsymbol{\Lambda}}}
        \mathbf R_t\mathbf R_t^\top
        +\mathbf S_t\mathbf L_t
        +\mathbf L_t\mathbf S_t
    \right),\\
    \frac{d}{dt}\mathbf Z_t
    &=\mathbf 0,
\end{split}
\end{equation}
with initial condition $\mathbf M(0)=\mathbf M_0^\ast$, where
\begin{equation}
\begin{split}
    \mathbf L_t
    &=
    -\operatorname{diag}
    \left(
        \mathbf R_t\mathbf R_t^\top
        \bar{\widetilde{\boldsymbol{\Lambda}}}
    \right),\\
    \mathbf H_t
    &=
    \left(1-\frac{\tau\eta_G}{2}\right)
    \mathbf R_t^\top
    \bar{\widetilde{\boldsymbol{\Lambda}}}\mathbf R_t
    -
    \left(1+\frac{\tau\eta_T}{2}\right)
    \mathbf Q_t^\top
    \bar{\boldsymbol{\Lambda}}\mathbf Q_t
    -
    \frac{\tau}{2}
    \left(\eta_G^2+\eta_T^2\right)\mathbf I_d .
\end{split}
\label{eq:Lt_Ht}
\end{equation}
\end{theorem}

\begin{proof}
A proof sketch is given in Appendix~\ref{sect:theorem_proof}.
\end{proof}

Theorem~\ref{thm:main_ode} is the central reduction of the paper.
Mixture labels, heterogeneous class moments, cross-feature
correlations, and non-zero class means introduce no additional
macroscopic state variables; they affect the limiting dynamics only
through the effective second moments
$\bar{\boldsymbol{\Lambda}}$ and
$\bar{\widetilde{\boldsymbol{\Lambda}}}$.

\subsection{Interpretation}

The effective moments weight the true--discriminator and generator--discriminator couplings, while $\mathbf H_t$ and $\mathbf L_t$ encode normalization, noise, and stability feedback. When $\bar{\boldsymbol{\Lambda}}$ and $\bar{\widetilde{\boldsymbol{\Lambda}}}$ commute, their common eigenbasis separates the transverse dynamics into mode-wise systems. When they do not commute, their relative eigenspaces couple these modes, so stability cannot be determined from the two spectra alone. Independently, low-rank structure can raise the leading effective eigenvalue and create a learnable collective direction. Sections~\ref{sec:conditions} and \ref{sec:feature_boosting} develop these two mechanisms.

\section{Fixed Points and Stability under Mismatch}
\label{sec:conditions}

The ODE~\eqref{eq:ode_limit} permits unequal and noncommuting effective
second moments,
\[
    \bar{\boldsymbol{\Lambda}}
    \neq
    \bar{\widetilde{\boldsymbol{\Lambda}}},
\]
as well as unequal noise levels. A common orthogonal change of latent
coordinates can diagonalize either second moment, but diagonalizes both
only when they commute. We write
\begin{equation}
    \nu
    :=
    \eta_\star^2
    :=
    \frac{\eta_{\mathrm T}^2+\eta_{\mathrm G}^2}{2}.
    \label{eq:eta_star}
\end{equation}
The transverse stability of the equilibrium manifold depends on the
noise levels through $\nu$. Away from this manifold, the nonlinear
dynamics retain separate dependence on $\eta_{\mathrm T}$ and
$\eta_{\mathrm G}$ through $\mathbf H_t$. Proofs are collected in
Appendix~\ref{sec:fixed_point_analysis}.

\subsection{Equilibrium manifold and transverse Jacobian}

Setting $\mathbf Q=\mathbf R=\mathbf 0$ gives
$\mathbf L=\mathbf 0$ and
$\mathbf H^\star=-\tau\nu\mathbf I_d$. Hence every feasible pair
$(\mathbf P,\mathbf S)$ defines an equilibrium.

\begin{lemma}[Equilibrium manifold]
\label{lem:manifold}
Let $\mathbf P^\star,\mathbf S^\star$ be such that the matrix in
\eqref{eq:macro_state_matrix}, with
$\mathbf Q=\mathbf R=\mathbf 0$ and $\mathbf Z=\mathbf I_d$, is a
feasible Gram matrix. Then
\[
    (\mathbf P^\star,\mathbf 0,\mathbf 0,\mathbf S^\star)
\]
is an equilibrium of~\eqref{eq:ode_limit}. These equilibria form a
manifold $\mathcal M$. Perfect recovery corresponds to
$(\mathbf P^\star,\mathbf S^\star)=(\mathbf I_d,\mathbf I_d)$, whereas
$(\mathbf 0,\mathbf I_d)$ is the canonical total-failure point.
\end{lemma}

Since $\mathbf L$ and the nonconstant terms of $\mathbf H$ are
quadratic in $(\mathbf Q,\mathbf R)$, perturbations of
$(\mathbf P,\mathbf S)$ tangent to $\mathcal M$ are neutral at first
order. Normal attraction is determined entirely by the
$(\delta\mathbf Q,\delta\mathbf R)$ subsystem.

\begin{proposition}[Exact transverse linearization]
\label{prop:transverse}
At
$(\mathbf P^\star,\mathbf 0,\mathbf 0,\mathbf S^\star)\in\mathcal M$,
each column of
$(\delta\mathbf Q,\delta\mathbf R)$ evolves under
\begin{equation}
    \boldsymbol{\mathcal J}
    (\mathbf P^\star,\mathbf S^\star)
    =
    \begin{pmatrix}
        \tau(\bar{\boldsymbol{\Lambda}}
             -\tau\nu\mathbf I_d)
        &
        -\tau\mathbf P^\star
             \bar{\widetilde{\boldsymbol{\Lambda}}}
        \\[2pt]
        \tau(\mathbf P^\star)^\top
             \bar{\boldsymbol{\Lambda}}
        &
        \widetilde{\tau}
             \bar{\widetilde{\boldsymbol{\Lambda}}}
        -\tau\mathbf S^\star
             \bar{\widetilde{\boldsymbol{\Lambda}}}
        -\tau^2\nu\mathbf I_d
    \end{pmatrix}.
    \label{eq:transverse_jacobian}
\end{equation}
Consequently, the full $2d^2$-dimensional transverse linearization has
the eigenvalues of $\boldsymbol{\mathcal J}$, each repeated $d$ times.
The equilibrium manifold is locally normally attracting at
$(\mathbf P^\star,\mathbf S^\star)$ if and only if
\begin{equation}
    \alpha\!\left(
        \boldsymbol{\mathcal J}
        (\mathbf P^\star,\mathbf S^\star)
    \right)<0,
    \qquad
    \alpha(\mathbf A):=
    \max_{\zeta\in\operatorname{spec}(\mathbf A)}
    \operatorname{Re}\zeta .
    \label{eq:spectral_abscissa_condition}
\end{equation}
No commutativity assumption is required.
\end{proposition}

Proposition~\ref{prop:transverse} is the general mismatch criterion.
Although it need not admit a scalar closed form, it gives an exact and
computable stability test for arbitrary true and generator second
moments.

\subsection{Onset of Learning}

At the canonical failure point
$(\mathbf P^\star,\mathbf S^\star)=(\mathbf 0,\mathbf I_d)$, the
transverse Jacobian is block diagonal:
\begin{equation}
\boldsymbol{\mathcal J}_{\mathrm{fail}}
=
\begin{pmatrix}
\tau(\bar{\boldsymbol{\Lambda}}-\tau\nu\mathbf I_d)
& \mathbf 0\\
\mathbf 0
&
(\widetilde{\tau}-\tau)
\bar{\widetilde{\boldsymbol{\Lambda}}}
-\tau^2\nu\mathbf I_d
\end{pmatrix}.
\label{eq:failure_jacobian}
\end{equation}
When $\widetilde{\tau}<\tau$, the lower block is strictly stable for
$\bar{\widetilde{\boldsymbol{\Lambda}}}\succeq\mathbf 0$ and $\nu>0$.
Failure is therefore transversely stable exactly when
\begin{equation}
\lambda_{\max}(\bar{\boldsymbol{\Lambda}})
<
\tau\nu
=
\frac{\tau}{2}
(\eta_{\mathrm T}^2+\eta_{\mathrm G}^2).
\label{eq:failure_stability}
\end{equation}
It is unstable when the inequality is reversed, while equality defines a
nonhyperbolic boundary. Thus, generator mismatch does not affect learning
onset in this regime, although it can change which points on the equilibrium
manifold are transversely stable after learning begins.

\subsection{Commuting mismatch}
\label{subsec:recovery}

Suppose
$\bar{\boldsymbol{\Lambda}}$ and
$\bar{\widetilde{\boldsymbol{\Lambda}}}$ commute. In a common
eigenbasis, write their paired eigenvalues as
$\lambda_i$ and $\widetilde\lambda_i$, and consider
\[
    \mathbf P^\star
    =
    \operatorname{diag}(p_1,\ldots,p_d),
    \qquad
    \mathbf S^\star=\mathbf I_d.
\]
The transverse Jacobian separates into mode-wise blocks
\begin{equation}
    \mathbf J_i(p_i)
    =
    \begin{pmatrix}
        \tau(\lambda_i-\tau\nu)
        &
        -\tau p_i\widetilde\lambda_i\\
        \tau p_i\lambda_i
        &
        (\widetilde\tau-\tau)\widetilde\lambda_i
        -\tau^2\nu
    \end{pmatrix}.
    \label{eq:mode_jacobian}
\end{equation}

\begin{proposition}[Recovery under commuting mismatch]
\label{prop:mismatch_stability}
Let $\lambda_i,\widetilde\lambda_i>0$. Mode $i$ is transversely stable
if and only if
\begin{align}
    &\tau\lambda_i
    +(\widetilde\tau-\tau)\widetilde\lambda_i
    <2\tau^2\nu,
    \label{eq:mismatch_trace}\\
    &\det\mathbf J_i(p_i)>0.
    \label{eq:mismatch_determinant}
\end{align}
For a learnable mode, $\lambda_i>\tau\nu$, the determinant condition is
equivalent to
\begin{equation}
    p_i^2>p_{c,i}^2
    :=
    \frac{\lambda_i-\tau\nu}{\tau\lambda_i}
    \left(
        \tau-\widetilde\tau
        +\frac{\tau^2\nu}{\widetilde\lambda_i}
    \right).
    \label{eq:recovery_stability}
\end{equation}
At perfect recovery, $p_i=1$, this additionally requires
\begin{equation}
    \widetilde\lambda_i
    >
    \widetilde\lambda_i^{\mathrm{crit}}
    :=
    \frac{
        \tau^2\nu(\lambda_i-\tau\nu)
    }{
        \widetilde\tau\lambda_i
        +\tau(\tau-\widetilde\tau)\nu
    }.
    \label{eq:generator_floor}
\end{equation}
\end{proposition}

The two conditions reveal an asymmetry between the true and generator
moments. The true moment determines whether a mode can escape failure,
whereas underestimating that mode in the generator can destabilize
perfect recovery. Moreover, $p_{c,i}^2$ decreases strictly with
$\widetilde\lambda_i$. Hence mismatch restricts the transversely stable
portion of $\mathcal M$ and can thereby restrict the possible limiting
alignments, without providing a global convergence guarantee.

\begin{corollary}[Matched solvable interval]
\label{cor:solvable}
If
$\widetilde\lambda_i=\lambda_i$ and
$\eta_{\mathrm T}=\eta_{\mathrm G}=\eta$, then for every learnable mode
$\lambda_i>\tau\eta^2$,
\begin{equation}
    \det\mathbf J_i(1)
    =
    \tau\widetilde\tau\lambda_i
    (\lambda_i-\tau\eta^2)
    +\tau^4\eta^4
    >0.
\end{equation}
The recovery condition therefore reduces to
$\widetilde\tau\lambda_i<2\tau^2\eta^2$. Combining recovery stability
with instability of failure yields
\begin{equation}
    \tau\eta^2
    <
    \lambda_i
    <
    \frac{2\tau^2\eta^2}{\widetilde\tau}.
    \label{eq:solvable_region}
\end{equation}
This interval is nonempty when $\widetilde\tau<2\tau$.
\end{corollary}

Consequently, learning begins when at least one mode satisfies
$\lambda_i>\tau\eta^2$, equivalently
$\lambda_{\max}>\tau\eta^2$. Learning every mode additionally requires
$\lambda_{\min}>\tau\eta^2$, while local stability of perfect recovery
requires
$\lambda_{\max}<2\tau^2\eta^2/\widetilde\tau$. These are local
stability statements, not basin-of-attraction guarantees.

\subsection{Partial-alignment branch}

Mismatch can also create equilibria outside $\mathcal M$ with
$r\neq0$, at which the discriminator remains aligned with a generator
direction. Consider one commuting eigendirection, write
$\lambda$ and $\widetilde\lambda$ for the corresponding moments, and
define
\[
    a_{\mathrm G}:=1-\frac{\tau\eta_{\mathrm G}}{2},
    \qquad
    a_{\mathrm T}:=1+\frac{\tau\eta_{\mathrm T}}{2},
    \qquad
    c:=\tau\nu+\widetilde\lambda-\lambda.
\]
Every nondegenerate scalar equilibrium with $r\neq0$ satisfies
\begin{equation}
\begin{aligned}
    &s=1,\qquad q=pr,\qquad
    H^\star=\widetilde\lambda-\lambda,\\
    &r^2
    =
    1-\frac{\tau\lambda}
             {\widetilde\tau\widetilde\lambda}
      (1-p^2),\\
    &r^2
    \left(
        a_{\mathrm G}\widetilde\lambda
        -a_{\mathrm T}\lambda p^2
    \right)
    =c.
\end{aligned}
\label{eq:branch_eqs}
\end{equation}

\begin{proposition}[Existence of a partial-alignment branch]
\label{prop:partial_branch}
Assume $\widetilde\tau<\tau$ and
$\lambda,\widetilde\lambda>0$.

\begin{enumerate}[nosep,leftmargin=*]
    \item If
    \[
        \widetilde\lambda
        <
        \lambda-\tau\nu,
    \]
    then~\eqref{eq:branch_eqs} has an admissible solution with
    $0<r^2<1$.

    \item If $c\geq0$ and
    \[
        a_{\mathrm G}\widetilde\lambda
        \leq
        a_{\mathrm T}\lambda
        \left(
            1-\frac{
                \widetilde\tau\widetilde\lambda
            }{\tau\lambda}
        \right),
    \]
    then no admissible $r\neq0$ solution exists.
\end{enumerate}
\end{proposition}

In the matched case, the second condition becomes
\[
    \frac{\widetilde\tau}{\tau}
    \leq
    \frac{2\tau\eta}{2+\tau\eta},
\]
under which the scalar partial-alignment branch is absent. Thus, the
absence of discriminator--generator locking in the earlier regime is
not generic: sufficiently strong generator underestimation can create
such a branch. The existence of this branch is checked numerically by direct substitution into the stationarity equations; its stability is not characterized here.

\subsection{Noncommuting mismatch and eigenspace alignment}
\label{subsec:noncommuting}

At perfect recovery,
$(\mathbf P^\star,\mathbf S^\star)=(\mathbf I_d,\mathbf I_d)$,
Proposition~\ref{prop:transverse} gives the exact criterion
\begin{equation}
    \alpha\!\left[
    \begin{pmatrix}
        \tau(\bar{\boldsymbol{\Lambda}}-\tau\nu\mathbf I_d)
        &
        -\tau\bar{\widetilde{\boldsymbol{\Lambda}}}\\
        \tau\bar{\boldsymbol{\Lambda}}
        &
        (\widetilde\tau-\tau)
        \bar{\widetilde{\boldsymbol{\Lambda}}}
        -\tau^2\nu\mathbf I_d
    \end{pmatrix}
    \right]<0.
    \label{eq:general_recovery_condition}
\end{equation}
A necessary, but generally insufficient, condition is
\begin{equation}
    \tau\operatorname{tr}
        \bar{\boldsymbol{\Lambda}}
    +(\widetilde\tau-\tau)
        \operatorname{tr}
        \bar{\widetilde{\boldsymbol{\Lambda}}}
    <2d\tau^2\nu.
    \label{eq:trace_necessary}
\end{equation}

The two spectra alone do not determine stability. For example, fix
\[
\bar{\boldsymbol{\Lambda}}
=\operatorname{diag}(0.8,1.8),
\qquad
\bar{\widetilde{\boldsymbol{\Lambda}}}(\theta)
=
\mathbf R_\theta
\operatorname{diag}(0.5,2.1)
\mathbf R_\theta^\top,
\]
with
$\tau=0.5$, $\widetilde\tau=0.1$, and
$\eta_{\mathrm T}=\eta_{\mathrm G}=1$. The spectral abscissa in
\eqref{eq:general_recovery_condition} changes from $-0.150$ at
$\theta=0^\circ$ to $0.378$ at $\theta=90^\circ$ and crosses zero at
$\theta_c\simeq33.7^\circ$. Perfect recovery is therefore destabilized
solely by changing the relative eigenbasis, while keeping both spectra
fixed. The endpoints commute and correspond to the two possible
eigenvalue pairings; intermediate angles isolate the genuinely
noncommuting effect.

\section{Correlation-Induced Signal Boosting}
\label{sec:feature_boosting}

Section~\ref{sec:conditions} shows that the onset of learning is
controlled by the largest eigenvalue of the true effective second
moment. We now identify a collective mechanism through which correlated
features become learnable even though every individual coordinate
remains below the onset threshold.

Consider
\begin{equation}
    \boldsymbol{\Lambda}(\theta)
    =
    \mathbf D
    +
    \theta\boldsymbol{\gamma}\boldsymbol{\gamma}^\top,
    \qquad
    \mathbf D=\operatorname{diag}(d_1,\ldots,d_d),
    \quad
    \theta\geq0,
    \label{eq:rank_one_spike}
\end{equation}
where $\boldsymbol{\gamma}\neq\mathbf 0$. For a unit vector
$\mathbf v$, its effective signal strength is the Rayleigh quotient
\begin{equation}
\begin{aligned}
    \sigma_{\mathrm{eff}}(\mathbf v;\theta)
    &:=
    \mathbf v^\top
    \boldsymbol{\Lambda}(\theta)\mathbf v\\
    &=
    \sum_{i=1}^{d}d_i v_i^2
    +
    \theta
    \langle\mathbf v,\boldsymbol{\gamma}\rangle^2.
\end{aligned}
\label{eq:effective_signal_strength}
\end{equation}
Maximizing over $\|\mathbf v\|=1$ gives
$\lambda_{\max}(\boldsymbol{\Lambda}(\theta))$.

Let
\[
    \widehat{\boldsymbol{\gamma}}
    :=
    \frac{\boldsymbol{\gamma}}
         {\|\boldsymbol{\gamma}\|},
    \qquad
    \bar d_{\boldsymbol{\gamma}}
    :=
    \widehat{\boldsymbol{\gamma}}^\top
    \mathbf D
    \widehat{\boldsymbol{\gamma}}.
\]

\begin{proposition}[Collective boosting and recovery instability]
\label{prop:boosting}
The following statements hold.

\begin{enumerate}[label=(\roman*),leftmargin=2em,itemsep=0.1em]

\item \emph{Collective boosting.}
Suppose
\begin{equation}
    \max_i\left(d_i+\theta\gamma_i^2\right)
    <
    \tau\nu
    <
    \bar d_{\boldsymbol{\gamma}}
    +\theta\|\boldsymbol{\gamma}\|^2.
    \label{eq:collective_boosting_condition}
\end{equation}
Then every coordinate direction remains individually below the onset
threshold, but
\[
    \lambda_{\max}
    (\boldsymbol{\Lambda}(\theta))
    >
    \tau\nu.
\]
Consequently, total failure is unstable and learning begins along a
non-coordinate direction produced by the correlated structure.

A sufficient range for~\eqref{eq:collective_boosting_condition} is
\begin{equation}
    \frac{
        \tau\nu-\bar d_{\boldsymbol{\gamma}}
    }{
        \|\boldsymbol{\gamma}\|^2
    }
    <
    \theta
    <
    \min_{i:\gamma_i\neq0}
    \frac{\tau\nu-d_i}{\gamma_i^2},
    \label{eq:boosting_theta_interval}
\end{equation}
provided this interval is nonempty.

\item \emph{Matched recovery instability.}
Suppose additionally that
\[
    \bar{\widetilde{\boldsymbol{\Lambda}}}
    =
    \boldsymbol{\Lambda}(\theta),
    \qquad
    \eta_{\mathrm T}=\eta_{\mathrm G}=\eta.
\]
If
\begin{equation}
    \bar d_{\boldsymbol{\gamma}}
    +\theta\|\boldsymbol{\gamma}\|^2
    \geq
    \frac{2\tau^2\eta^2}{\widetilde\tau},
    \label{eq:boosting_instability}
\end{equation}
then
\[
    \lambda_{\max}
    (\boldsymbol{\Lambda}(\theta))
    \geq
    \frac{2\tau^2\eta^2}{\widetilde\tau}.
\]
Perfect recovery therefore fails to be strictly transversely stable.
It is unstable when the inequality is strict and lies on the
nonhyperbolic boundary at equality.

\end{enumerate}
\end{proposition}

\begin{proof}
Evaluating~\eqref{eq:effective_signal_strength} at
$\widehat{\boldsymbol{\gamma}}$ gives
\[
    \lambda_{\max}
    (\boldsymbol{\Lambda}(\theta))
    \geq
    \bar d_{\boldsymbol{\gamma}}
    +\theta\|\boldsymbol{\gamma}\|^2.
\]
For part (i), this lower bound exceeds $\tau\nu$, while
\[
    \mathbf e_i^\top
    \boldsymbol{\Lambda}(\theta)\mathbf e_i
    =
    d_i+\theta\gamma_i^2
    <
    \tau\nu
\]
for every coordinate. Part (ii) follows from the upper boundary in
Corollary~\ref{cor:solvable}.
\end{proof}

Proposition~\ref{prop:boosting} identifies a bounded-benefit regime.
Moderate correlation can create a learnable collective direction while
all coordinate-wise marginal strengths remain subthreshold. In the
matched setting, however, an excessively strong spike can cross the
upper recovery boundary. For mismatched generator moments, the onset
conclusion remains unchanged, but recovery stability cannot be inferred
from $\lambda_{\max}(\boldsymbol{\Lambda})$ alone and must instead be
evaluated using~\eqref{eq:general_recovery_condition}.

In the structured mixture model, the same Rayleigh-quotient argument
applies directly to
$\bar{\boldsymbol{\Lambda}}$. Class-dependent correlations and non-zero
means may contribute several low-rank components, whose combined
eigenstructure determines whether collective signal boosting occurs.

\section{Experimental Results}
\label{sec:experiments}

We evaluate the theory in two complementary settings. First,
controlled synthetic experiments compare the macroscopic ODE with
high-dimensional microscopic SGDA and test the predicted spectral
onset and generator-mismatch effects. Second, experiments on MNIST,
FashionMNIST, and CIFAR-10 examine whether sensitivity to the generator
second moment remains visible outside the exact spiked-model
assumptions. These image experiments are intended as controlled
stress tests of the solvable model, not as benchmarks of practical
image generation.

\subsection{Controlled ODE--SGDA Comparisons}

We compare the macroscopic ODE with microscopic SGDA at $n=1500$ using
\begin{equation}
\operatorname{ov}(\mathbf U,\mathbf V)
:=
d^{-1}\operatorname{tr}
\!\left(\mathbf P\mathbf S^\dagger\mathbf P^\top\right),
\label{eq:subspace_overlap}
\end{equation}
the mean squared cosine of the principal angles when $\mathbf V$ has full
column rank. Unless stated otherwise, $d=2$, $\tau=0.1$, and
$\widetilde\tau=0.01$.

\begin{figure}[tbp]
    \centering
    \begin{subfigure}[b]{0.31\textwidth}
        \centering
        \includegraphics[width=\textwidth]
        {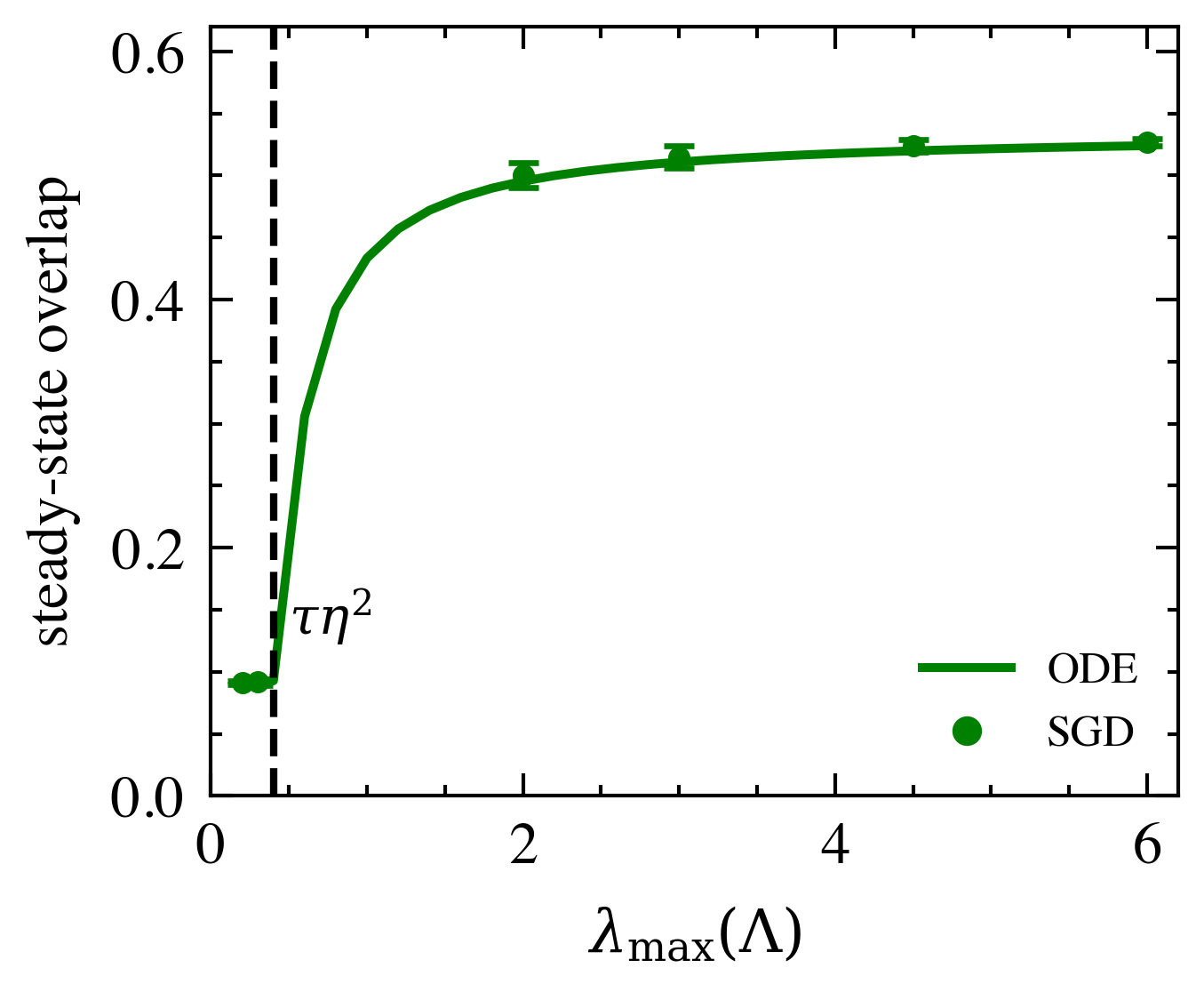}
        \caption{}\label{fig:validation_onset}
    \end{subfigure}\hfill
    \begin{subfigure}[b]{0.31\textwidth}
        \centering
        \includegraphics[width=\textwidth]
        {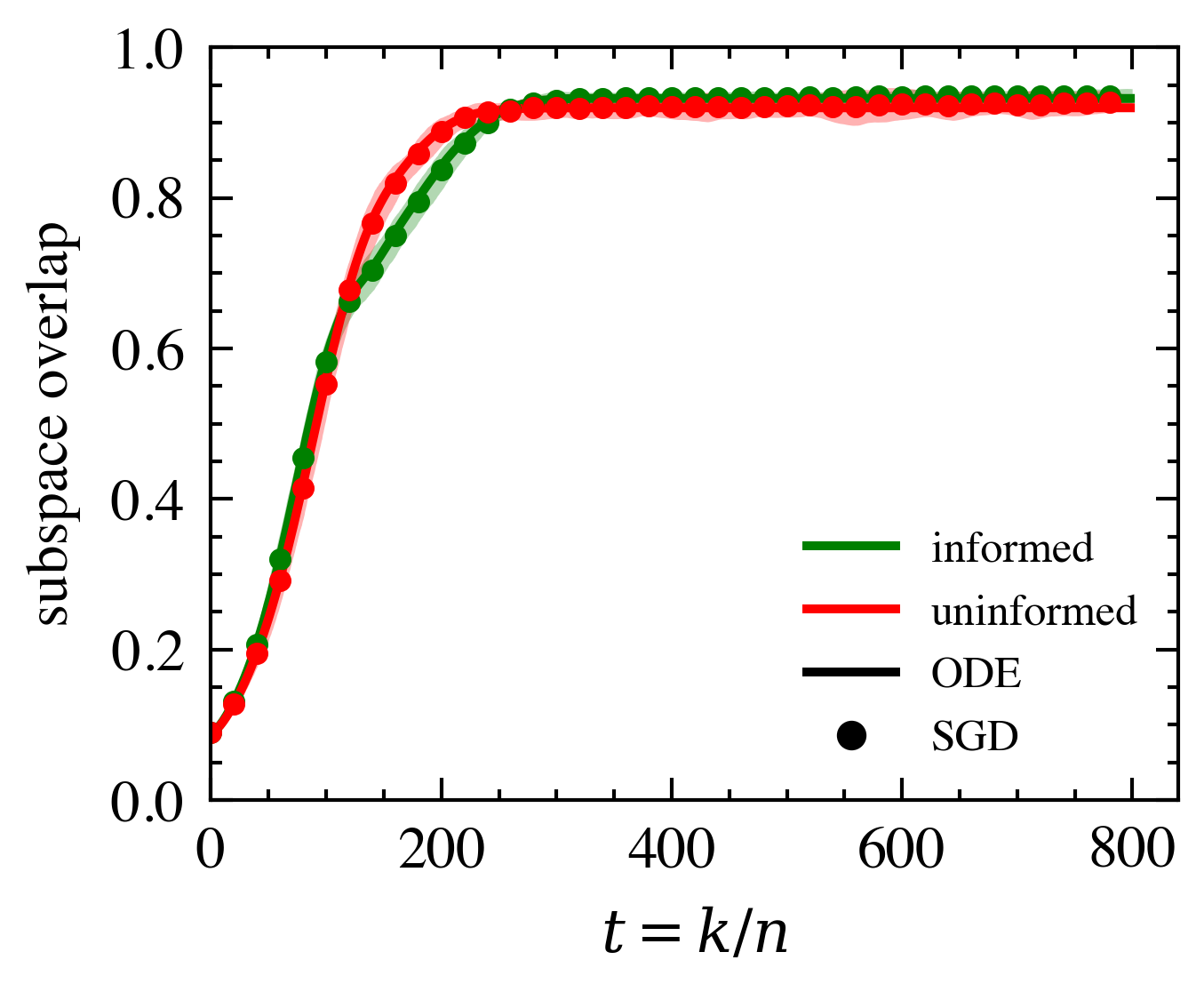}
        \caption{}\label{fig:validation_mismatch}
    \end{subfigure}\hfill
    \begin{subfigure}[b]{0.31\textwidth}
        \centering
        \includegraphics[width=\textwidth]
        {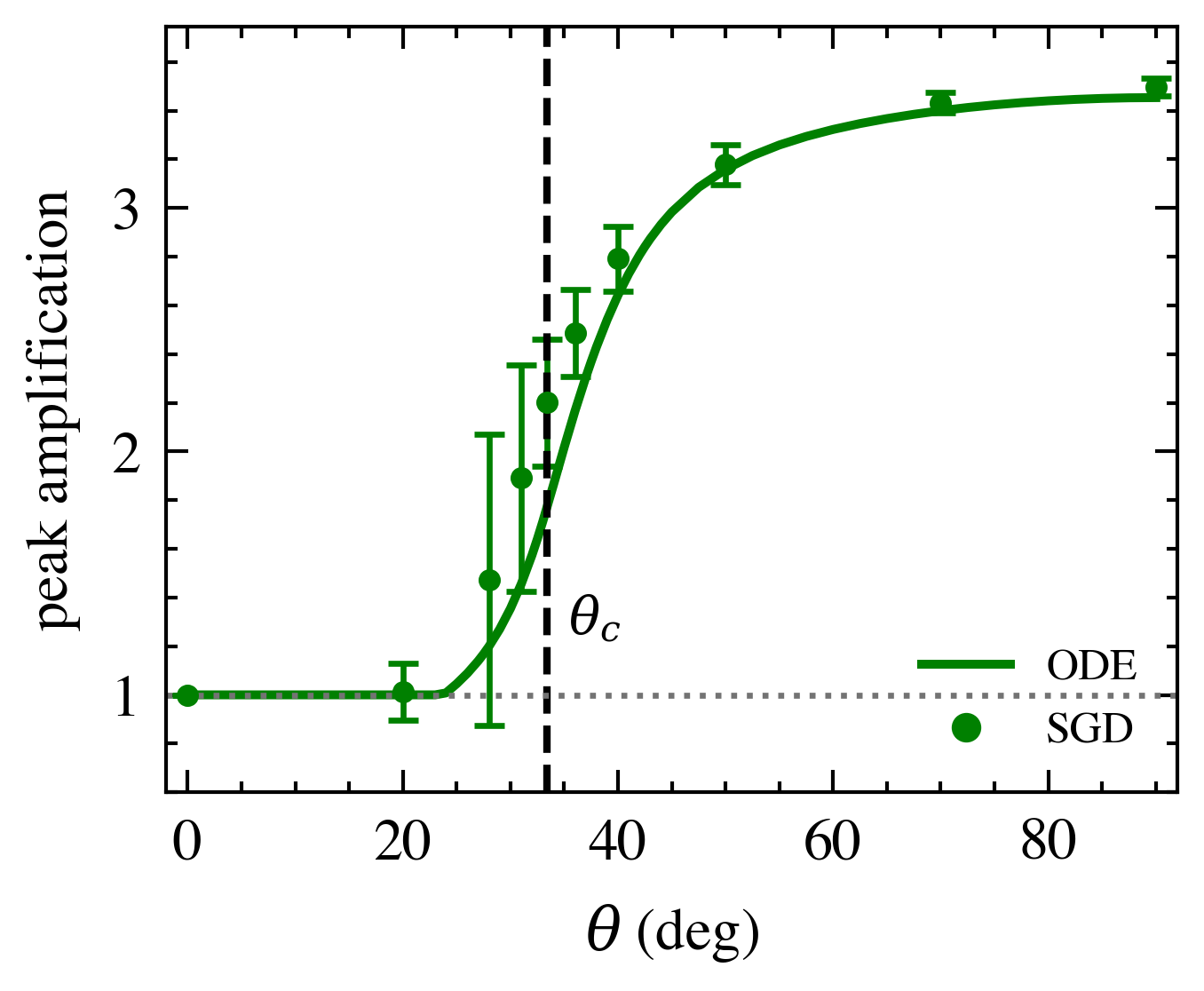}
        \caption{}\label{fig:validation_rotation}
    \end{subfigure}
    \caption{ODE trajectories (curves) and SGDA simulations (markers);
    markers show ten-seed means with $3\sigma$ error bars.
    \textbf{(a)} Spectral onset for
    $\boldsymbol{\Lambda}
    =0.2\mathbf I+\theta\boldsymbol{\gamma}\boldsymbol{\gamma}^\top$.
    \textbf{(b)} Matched
    $\bar{\widetilde{\boldsymbol{\Lambda}}}
    =\boldsymbol{\Lambda}=\operatorname{diag}(2,4)$ versus the
    matched-trace isotropic moment $3\mathbf I$.
    \textbf{(c)} Peak transverse amplification under fixed-spectrum
    noncommuting mismatch. The dashed line is the predicted
    $\alpha(\mathbf J)=0$ boundary; nonnormality can produce transient
    amplification on its stable side. Panels (a)--(b) use $n=1500$,
    while panel (c) uses $n=12{,}000$.}
    \label{fig:validation}
\end{figure}

\paragraph{Equivalent latent mixtures.}
We compare three latent laws with
$\bar{\boldsymbol{\Lambda}}=\operatorname{diag}(1,3)$: a two-class
zero-mean mixture, a four-class mixture with non-zero means and opposing
rank-one correlations along $(1,\pm1)/\sqrt2$, and an unconditional
Gaussian. At $n=1500$, their mean overlap trajectories differ by at most
$0.0039$ over five seeds, within the pooled $3\sigma$ band of $0.035$.
This supports the reduction and its identifiability limitation
(Remark~\ref{rem:identifiability}).

\paragraph{Spectral onset.}
Figure~\ref{fig:validation}(\subref{fig:validation_onset}) confirms~\eqref{eq:failure_stability}:
the aggregate overlap increases once
$\lambda_{\max}(\boldsymbol{\Lambda})>\tau\eta^2$. This establishes
learning onset in at least one effective direction, not recovery of every
mode.

\paragraph{Generator-moment mismatch.}
Figure~\ref{fig:validation}(\subref{fig:validation_mismatch}) compares a matched moment with an isotropic
moment of equal trace. Their different trajectories and displayed-horizon
values are consistent with Proposition~\ref{prop:mismatch_stability}, which
shows that generator underestimation reduces the transversely stable range.
The panel demonstrates finite-time sensitivity to the generator moment, not
convergence to distinct limiting subspaces.

\paragraph{Noncommuting mismatch.}
Figure~\ref{fig:validation}(\subref{fig:validation_rotation}) uses
\[
\bar{\boldsymbol{\Lambda}}=\operatorname{diag}(0.8,1.8),\qquad
\bar{\widetilde{\boldsymbol{\Lambda}}}(\theta)
=\mathbf R_\theta\operatorname{diag}(0.5,2.1)\mathbf R_\theta^\top,
\]
with $\tau=0.5$, $\widetilde\tau=0.1$, and
$\eta_{\mathrm T}=\eta_{\mathrm G}=1$. Starting near perfect recovery
with
$\|(\delta\mathbf Q_0,\delta\mathbf R_0)\|_{\mathrm F}
=\varepsilon=0.31$, we integrate to $T_{\rm amp}=300$ and report
\[
A_{T_{\rm amp}}(\theta)
:=
\max_{0\leq t\leq T_{\rm amp}}
\frac{\bigl(\|\mathbf Q_t\|_{\mathrm F}^2+
\|\mathbf R_t\|_{\mathrm F}^2\bigr)^{1/2}}
{\bigl(\|\mathbf Q_0\|_{\mathrm F}^2+
\|\mathbf R_0\|_{\mathrm F}^2\bigr)^{1/2}}.
\]
We use $n=12{,}000$ for SGDA so that the seeded perturbation remains
distinguishable from finite-$n$ fluctuations and the trajectory more
closely tracks the ODE. The transverse Jacobian predicts
$\theta_c\simeq33.7^\circ$, where $\alpha(\mathbf J)$ crosses zero.
For $\theta<\theta_c$, the equilibrium is locally asymptotically stable,
although nonnormality can produce transient amplification; for
$\theta>\theta_c$, it is transversely unstable. Thus, the dashed
boundary is the local linear prediction, whereas
$A_{T_{\rm amp}}$ measures finite-amplitude behavior. The ODE and SGDA
closely agree across the rotation sweep.

\subsection{Image experiments}
\label{subsec:image_experiments}

For real data, the true signal subspace is unknown. We therefore use a
top-$d$ PCA subspace as a common reference for comparing generator
runs. This reference is a reproducible surrogate and is not treated as
ground-truth access to $\mathbf U$. All Grassmann distances are
computed between subspaces, after orthonormalizing the learned
generator basis for evaluation.

Images are centred using the training-set mean. Consequently, the
aggregate empirical mean is zero, while individual classes may retain
non-zero projected means. The informed construction uses per-class PCA
spectra in the generator latent coordinates. Labels select the latent
mixture component but are not supplied as inputs to the discriminator.

\begin{figure}[tbp]
    \centering
    \resizebox{\linewidth}{!}{%
    \begin{tikzpicture}
    \begin{groupplot}[
      group style={
        group size=3 by 1,
        horizontal sep=12pt
      },
      width=5.3cm,
      height=4.3cm,
      xmin=0,
      ymin=0.4,
      ymax=17,
      xlabel={SGD iteration ($\times10^3$)},
      tick label style={font=\scriptsize},
      label style={font=\scriptsize},
      title style={font=\scriptsize},
      legend style={
        at={(0.02,0.975)},
        anchor=north west,
        draw=none,
        fill=white,
        fill opacity=0.82,
        text opacity=1,
        font=\tiny,
        row sep=-2.5pt,
        inner xsep=1.5pt,
        inner ysep=1pt
      },
      legend cell align=left,
      legend image code/.code={
        \draw[#1,line width=0.9pt]
        (0cm,0cm)--(0.24cm,0cm);
      },
      every axis plot/.append style={
        line width=1.0pt,
        line join=round,
        mark=none
      }]

      \nextgroupplot[
        title={MNIST},
        xmax=120,
        ylabel={$d(\mathbf V,\mathbf U)$}
      ]
        \addplot[cinf] coordinates {\datMnistinformed};
        \addlegendentry{informed}
        \addplot[cunf] coordinates {\datMnistuninformed};
        \addlegendentry{uninformed}
        \node[anchor=north east,inner sep=0pt]
        at (rel axis cs:0.975,0.965)
        {\setlength{\fboxsep}{0.6pt}
        \fcolorbox{black!45}{white}{
        \includegraphics[width=1.35cm]{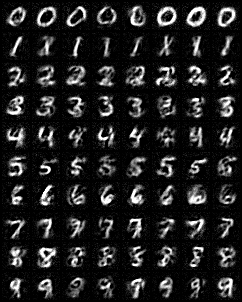}}};

      \nextgroupplot[
        title={FashionMNIST},
        xmax=120
      ]
        \addplot[cinf] coordinates {\datFmnistinformed};
        \addplot[cunf] coordinates {\datFmnistuninformed};
        \node[anchor=north east,inner sep=0pt]
        at (rel axis cs:0.975,0.965)
        {\setlength{\fboxsep}{0.6pt}
        \fcolorbox{black!45}{white}{
        \includegraphics[width=1.35cm]{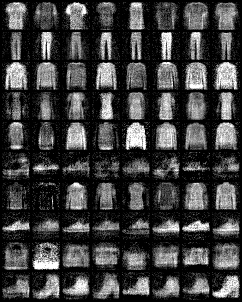}}};

      \nextgroupplot[
        title={CIFAR-10},
        xmax=120
      ]
        \addplot[cinf] coordinates {\datCifarinformed};
        \addplot[cunf] coordinates {\datCifaruninformed};
        \node[anchor=north east,inner sep=0pt]
        at (rel axis cs:0.975,0.965)
        {\setlength{\fboxsep}{0.6pt}
        \fcolorbox{black!45}{white}{
        \includegraphics[width=1.35cm]{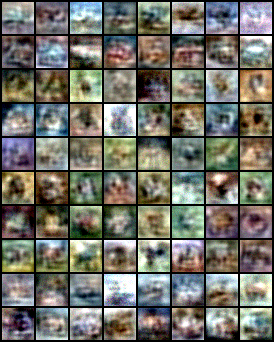}}};

    \end{groupplot}
    \end{tikzpicture}%
    }

    \caption{
    Informed-versus-uninformed comparison on MNIST, FashionMNIST, and
    CIFAR-10 ($d=64$). Curves show the mean Grassmann distance over ten
    runs to a held-out PCA reference. Informed moments use per-class
    PCA spectra; uninformed strengths are sampled uniformly from
    $[0.5,5]$. Insets show class-structured samples from the informed
    linear generator and are included only as qualitative illustrations.
    }
    \label{fig:grassmann_3data}
\end{figure}

Figure~\ref{fig:grassmann_3data} shows the same qualitative pattern on
MNIST~\citep{mnist}, FashionMNIST~\citep{fashionmnist}, and
CIFAR-10~\citep{cifar}: the informed construction initially improves
more slowly but attains lower Grassmann distance at the displayed
endpoint. The reference is constructed from a held-out split
disjoint from the training split used to fit the informed moment, the
same protocol as Table~\ref{tab:e6}.

\paragraph{Held-out MNIST comparison.}
For Table~\ref{tab:e6}, the informed moment is fitted on a training split,
whereas evaluation uses a disjoint held-out split. Let
$\{\phi_j\}_{j=1}^d$ be the principal angles between the learned and
reference subspaces. We report
\[
d_{\mathrm{Gr}}
:=\Bigl(\sum_{j=1}^d\phi_j^2\Bigr)^{1/2},
\qquad
e_{\mathrm{cov}}
:=
\frac{\|\widehat{\boldsymbol{\Sigma}}_{\mathrm{held}}
-\widehat{\boldsymbol{\Sigma}}_{\mathrm G}\|_{\mathrm F}}
{\|\widehat{\boldsymbol{\Sigma}}_{\mathrm{held}}\|_{\mathrm F}},
\]
where $\widehat{\boldsymbol{\Sigma}}_{\mathrm G} = V \Sigma_{eff} V^\top $ is the covariance induced
by the generator at the reported iterate. The baselines are matched-trace
isotropic, shuffled-spectrum, shrunk-spectrum, uniform, and class-agnostic
global PCA.

\begin{table}[tbp]
\centering
\small
\caption{Generator-moment controls on MNIST
($d=64$, $120{,}000$ iterations). Entries are mean $\pm$ s.d.\ over ten
runs and are rounded to three decimals, where shuffled and uniform are resampled per run.}
\label{tab:e6}
\begin{tabular}{@{}lccc@{}}
\toprule
Generator second moment
& $d_{\mathrm{Gr}}\downarrow$
& $e_{\mathrm{cov}}\downarrow$
& $\|\mathbf R\|_{\mathrm F}/\sqrt d$\\
\midrule
Informed (per-class PCA)
& $\mathbf{0.898\pm0.035}$
& $0.724\pm0.002$
& $0.972\pm0.000$\\
Isotropic, matched trace
& $1.025\pm0.011$
& $0.924\pm0.000$
& $0.819\pm0.000$\\
Shuffled spectrum
& $1.067\pm0.084$
& $0.926\pm0.010$
& $0.845\pm0.021$\\
Uniformly shrunk ($\times0.25$)
& $1.584\pm0.195$
& $0.933\pm0.002$
& $0.979\pm0.003$\\
Uniform $[0.5,5]$
& $1.911\pm0.065$
& $2.027\pm0.052$
& $0.133\pm0.002$\\
Global (non-class) PCA
& $2.701\pm0.023$
& $\mathbf{0.479\pm0.008}$
& $0.591\pm0.003$\\
\bottomrule
\end{tabular}
\end{table}

At the reported horizon, the informed construction gives the smallest
Grassmann distance. Its improvement is $1.013$ over the original uniform
baseline but only $0.127$ over the stronger matched-trace isotropic control;
the latter is the more informative estimate of the benefit of class-specific
spectral information. The metrics need not rank the controls identically:
global PCA gives the smallest held-out covariance-fit error but the largest
Grassmann distance, showing that covariance fit and finite-time subspace
alignment measure different properties.

The ordering is horizon dependent and reverses for some baselines before
approximately $60{,}000$ iterations. Moreover, the residual
$\|\mathbf R\|_{\mathrm F}/\sqrt d$ values show that several runs retain
generator--discriminator coupling. Table~\ref{tab:e6} therefore supports a
finite-horizon effect of the generator moment, not convergence to different
points of the equilibrium manifold. The sample insets in
Figure~\ref{fig:grassmann_3data} are qualitative illustrations only; we make
no claim of competitive image quality.

\section{Discussion and Conclusion}
\label{sec:conclusion}

We derived exact high-dimensional dynamics for sample-wise
multi-feature GAN training with class-structured, correlated, and
non-zero-mean latent mixtures. For the quadratic energy discriminator,
this structure enters the limiting ODE through the true and generator
effective second moments
$\bar{\boldsymbol{\Lambda}}$ and
$\bar{\widetilde{\boldsymbol{\Lambda}}}$. The central result is that the macroscopic observables of the generalized microscopic process still concentrate onto a closed deterministic trajectory.

Our main extension concerns mismatched moments. The equilibria form a
manifold whose transverse stability is characterized by a
$2d\times2d$ block Jacobian. In the commuting case, this yields explicit
mode-wise conditions and recovers the matched solvable interval as a
corollary. In the noncommuting case, stability also depends on the
relative eigenspaces and can change while both spectra remain fixed.
Generator underestimation can destabilize recovery and create
partial-alignment equilibria. Low-rank correlation can likewise boost
weak coordinates above the onset threshold, although excessive
amplification can destabilize recovery in the matched setting.

Simulations support the ODE approximation, spectral onset, and
effective-second-moment equivalence. In the image experiments, informed
class-specific moments achieve better finite-time subspace alignment
than the tested controls, although the improvement over matched-trace
isotropic moments is modest. Since several configurations retain
generator--discriminator alignment, these results should not be
interpreted as evidence of distinct asymptotic limits.

\textbf{Limitations.}
The analysis is restricted to a linear generator, quadratic
discriminator, fixed latent dimension, isotropic observation noise, and
plain two-timescale SGDA. The discriminator identifies only the
aggregate second moment, and the stability results are local rather
than global convergence guarantees. The model is therefore intended to
isolate exact trajectory-level mechanisms, not to reproduce practical
deep-GAN training. Extending these dynamics to nonlinear architectures,
higher-order objectives, and adaptive optimization is left for future
work.

\backmatter

\bmhead{Acknowledgements}
This work was supported by T\"UB\.{I}TAK 2232 Program (No.~118C337) and project 124E063, and by the KUIS AI Research Center. 

\section*{Statements and Declarations}
\textbf{Funding:} See Acknowledgements. \textbf{Competing interests:} None. \textbf{Data availability:} MNIST, FashionMNIST, and CIFAR-10 are publicly available. \textbf{Code availability:} Code for the ODE, microscopic simulations, stability analysis, preprocessing, and experiments will be available at
\url{https://github.com/KU-MLIP/}.

\begin{appendices}

\section{Details of the Stability Analysis}
\label{sec:fixed_point_analysis}

\paragraph{Equilibrium manifold and transverse linearization.}
At $\mathbf Q=\mathbf R=\mathbf0$, \eqref{eq:Lt_Ht} gives
$\mathbf L=\mathbf0$ and
$\mathbf H^\star=-\tau\nu\mathbf I_d$. Hence every feasible
$(\mathbf P,\mathbf S)$ is an equilibrium, proving
Lemma~\ref{lem:manifold}. Since $\mathbf L$ and the nonconstant terms
of $\mathbf H$ are quadratic in $(\mathbf Q,\mathbf R)$,
$(\delta\mathbf P,\delta\mathbf S)$ are neutral at first order.
Moreover,
$\delta\mathbf Q\mathbf H^\star=-\tau\nu\delta\mathbf Q$ and
$\delta\mathbf R\mathbf H^\star=-\tau\nu\delta\mathbf R$.
The remaining terms give~\eqref{eq:transverse_jacobian}, independently
on every column, so each of its eigenvalues has multiplicity $d$ in
the full transverse system.

\paragraph{Proof of Proposition~\ref{prop:mismatch_stability}.}
For~\eqref{eq:mode_jacobian},
$\operatorname{tr}\mathbf J_i(p_i)
=\tau\lambda_i+(\widetilde\tau-\tau)\widetilde\lambda_i-2\tau^2\nu$,
while
$\det\mathbf J_i(p_i)
=\tau(\lambda_i-\tau\nu)
[(\widetilde\tau-\tau)\widetilde\lambda_i-\tau^2\nu]
+\tau^2p_i^2\lambda_i\widetilde\lambda_i$.
A real $2\times2$ matrix is Hurwitz exactly when its trace is negative
and its determinant positive. For $\lambda_i>\tau\nu$, solving the
second condition for $p_i^2$ gives~\eqref{eq:recovery_stability};
setting $p_i=1$ and rearranging gives~\eqref{eq:generator_floor}.
In the matched case,
$\widetilde\lambda_i=\lambda_i$ and $\nu=\eta^2$, and
$\det\mathbf J_i(1)
=\tau\widetilde\tau\lambda_i(\lambda_i-\tau\eta^2)
+\tau^4\eta^4>0$
for every learnable mode. The trace condition then becomes
$\widetilde\tau\lambda_i<2\tau^2\eta^2$, proving
Corollary~\ref{cor:solvable}.

\paragraph{Proof of Proposition~\ref{prop:partial_branch}.}
In one commuting eigendirection,
$L=-r^2\widetilde\lambda$ and
$H=a_{\mathrm G}r^2\widetilde\lambda
-a_{\mathrm T}q^2\lambda-\tau\nu$, where
$a_{\mathrm G}=1-\tau\eta_{\mathrm G}/2$ and
$a_{\mathrm T}=1+\tau\eta_{\mathrm T}/2$.
For $r\neq0$, stationarity forces $s=1$ and $q=pr$. The case $p=0$
gives either $r^2>1$ or $r^2\leq0$ and is inadmissible. For $p\neq0$,
stationarity further gives
$H^\star=\widetilde\lambda-\lambda$, after which the remaining
conditions reduce to~\eqref{eq:branch_eqs}.

Set $u=p^2$,
$\kappa=\tau\lambda/
(\widetilde\tau\widetilde\lambda)$,
$c=\tau\nu+\widetilde\lambda-\lambda$, and
$f(u)=[1-\kappa(1-u)]
[a_{\mathrm G}\widetilde\lambda-a_{\mathrm T}\lambda u]-c$.
Admissible equilibria are roots of $f$ in
$(\max\{0,1-\kappa^{-1}\},1]$. If $c<0$, then
$\kappa>1$ under $\widetilde\tau<\tau$, and
$f(1-\kappa^{-1})=-c>0$, whereas
$f(1)=-\tau(\eta_{\mathrm G}\widetilde\lambda
+\eta_{\mathrm T}\lambda)/2-\tau\nu<0$.
Continuity gives an admissible root.

Conversely, $f$ is concave. The second condition of
Proposition~\ref{prop:partial_branch} is equivalent to placing its
vertex at or to the left of $1-\kappa^{-1}$. If also $c\geq0$, then
$f$ is nonpositive and decreasing throughout the admissible interval,
so no admissible root exists.

\noindent\textbf{Numerical checks.} We checked the transverse Jacobian by finite differences for
noncommuting positive-definite moments, the commuting reduction against
the eigenvalues of~\eqref{eq:mode_jacobian}, and the partial-branch
roots by direct substitution into the stationarity equations. The
scripts are included with the released code.

\section{Proof of the Main Theorem}
\label{sect:theorem_proof}
\vspace{-1em}
We adapt the stochastic-approximation argument of
\citet{solvable,multifeaturegan}. Let $\mathbb E_k$ denote conditional
expectation given the state at iteration $k$, and decompose
\begin{equation}
\mathbf M_{k+1}-\mathbf M_k
=
\frac1n\phi(\mathbf M_k)
+\boldsymbol\rho_k+\boldsymbol\delta_k,
\label{eq:M_decomposition}
\end{equation}
where
$\boldsymbol\rho_k
=\mathbf M_{k+1}-\mathbb E_k\mathbf M_{k+1}$,
$\boldsymbol\delta_k
=\mathbb E_k\mathbf M_{k+1}-\mathbf M_k
-n^{-1}\phi(\mathbf M_k)$, and $\phi$ is the right-hand side
of~\eqref{eq:ode_limit}.

A standard stochastic-approximation result
\citep{subspace-incomplete} gives uniform convergence on $[0,T]$ at
rate $n^{-1/2}$ when the partial sums of $\boldsymbol\rho_k$ form a
martingale,
$\mathbb E\|\boldsymbol\rho_k\|_{\mathrm F}^2=O(n^{-2})$,
$\mathbb E\|\boldsymbol\delta_k\|_{\mathrm F}=O(n^{-3/2})$, the drift
is Lipschitz on the feasible state space, and the initialization error
is $O(n^{-1/2})$ in expected Frobenius norm.

The column constraints make the feasible Gram matrices compact, with
all entries bounded in magnitude by one. Hence the polynomial drift
$\phi$ is Lipschitz on this set and the macroscopic moments are
uniformly bounded. Assumption (A.5) gives the required initialization
rate. The smooth normalization maps in~\eqref{eq:sgda_updates} admit
second-order expansions whose leading corrections are $\mathbf L_k$
and $\mathbf H_k$. Using these expansions, the row-wise delocalization
condition (A.4), and the moment estimates of
\citet{solvable,multifeaturegan} give
$\mathbb E\|\boldsymbol\rho_k\|_{\mathrm F}^2\leq C(T)n^{-2}$ and
$\mathbb E\|\boldsymbol\delta_k\|_{\mathrm F}\leq C(T)n^{-3/2}$.

It remains to identify the modified drift. For the structured mixture,
$\mathbb E[\mathbf c_k\mathbf c_k^\top]
=\sum_{\ell=1}^{L}\pi_\ell
\mathbb E[\mathbf c_k\mathbf c_k^\top\mid\ell_k=\ell]
=\bar{\boldsymbol{\Lambda}}$, and analogously
$\mathbb E[\widetilde{\mathbf c}_k
\widetilde{\mathbf c}_k^\top]
=\bar{\widetilde{\boldsymbol{\Lambda}}}$.
The finite Gaussian mixture has bounded moments of every order, while
independence and zero-mean observation noise eliminate latent--noise
cross terms. Substitution into the conditional generator and
discriminator updates gives the drift in~\eqref{eq:ode_limit}, with
$O_{L^1}(n^{-3/2})$ remainder. In particular, the discriminator
normalization contributes $\mathbf Q_k\mathbf H_k$ and
$\mathbf R_k\mathbf H_k$, because its microscopic correction acts as
$\mathbf W_k\mathbf H_k$.

Thus, higher latent moments enter only the martingale and remainder
bounds, whereas the leading drift depends on the latent laws through
$\bar{\boldsymbol{\Lambda}}$ and
$\bar{\widetilde{\boldsymbol{\Lambda}}}$. Applying the
stochastic-approximation result to~\eqref{eq:M_decomposition} yields
the $C(T)/\sqrt n$ bound in~\eqref{eq:ode_convergence}, proving
Theorem~\ref{thm:main_ode}.
\end{appendices}
\vspace{-1em}
\bibliography{sn-bibliography}

\end{document}